\documentclass{llncs}
\usepackage{algorithm}
\usepackage{algcompatible}
\usepackage{algpseudocode}
\usepackage{pifont}
\usepackage{times}
\usepackage{epsfig}
\usepackage{graphicx}
\usepackage{mathtools}
\usepackage{amssymb}
\usepackage{enumerate}

\graphicspath{ {/Users/ashu/Desktop/IISc/} }

\def\R{{\mathbb{R}}}
\def\E{{\mathbb{E}}}
\def\H{{\mathcal{H}}}
\def\O{{\mathcal{O}}}

\newcommand\norm[1]{\left\lVert#1\right\rVert}

\newcommand{\remove}[1]{}

\begin{document}
\title{Accelerated Randomized Coordinate Descent Algorithms for Stochastic Optimization and Online Learning}

\author{Akshita Bhandari\and Chandramani Singh}
\institute{Department of ESE, Indian Institute of Science Bangalore, India \\
\email{{{akshita,chandra}}@iisc.ac.in}}


\maketitle

\begin{abstract}
   We propose accelerated randomized coordinate descent algorithms for stochastic optimization and online learning. Our algorithms have significantly less per-iteration complexity than the known accelerated gradient algorithms. The proposed algorithms for online learning have better regret performance than the known 
   randomized online coordinate descent algorithms. Furthermore, the proposed algorithms
   for stochastic optimization exhibit as good convergence rates as the best known randomized coordinate descent algorithms. We also show simulation results to demonstrate performance of the proposed algorithms.
\end{abstract}

\section{Introduction}
Convex optimization problems are at the heart of many machine learning algorithms.
These problems are typically huge in scale due to either large number of samples
or large number of features or both. Gradient descent type algorithms are standard
approaches to solve these problems. However, these algorithms are computationally 
expensive, i.e., have huge per-iteration complexity, in large scale problems. 
We thus require alternative iterative algorithms where
\begin{enumerate}
\item Only one feature parameter (or, a block of parameters) is updated
at each iteration. These are coordinate descent algorithms~\cite{IEEEhowto:kopka1}; 
these check per-iteration computation complexity when the number of features
are humongous.
\item Data samples are processed one~(or, one block) at a time. This is preferred 
when data sets have enormous samples, and is necessitated in  
scenario where samples are availed in real time. Processing randomly chosen samples  
from the whole available data set leads to stochastic gradient type algorithms~\cite{IEEEhowto:kopka2}, whereas
processing samples in real time is referred to as 
online convex optimization or online learning~\cite{IEEEhowto:kopka3}.   
\end{enumerate}
None of these alternatives is adequate if both, the number of
samples as well as the number of features, are huge.

We propose iterative descent algorithms that first choose a sample and then
randomly choose a feature, and update the corresponding parameter only based 
on the chosen sample. In other words, we propose algorithms that
combine characteristics of coordinate descent and stochastic gradient descent~(or, online 
gradient descent) algorithms. It is well known that stochastic coordinate
descent algorithms suffer from slow convergence due to 
variance of stochastic gradients of random samples. On the other hand, 
randomized coordinate descent algorithms have same convergence rate
as gradient descent but worse constants. We have proposed ``accelerated''
gradient algorithms in order to alleviate these deficiencies. 

\subsection{Related Work}
Stochastic gradient descent~(SGD)~\cite{IEEEhowto:kopka2} was introduced by {Nemirovski et al.} where as
Online convex optimization and the associated projected gradient~(OGD)~\cite{IEEEhowto:kopka3} were
introduced by ~{Zinkevich}. Hu et al. derived accelerated versions of SGD and OGD; they refer 
to these algorithms as Stochastic Accelerated GradiEnt~(SAGE)~\cite{IEEEhowto:kopka4}.
 Recently, Roux et al. proposed stochastic average gradient~(SAG)~\cite{IEEEhowto:kopka5} and 
 Johnson and Zhang proposed Stochastic variation reduced gradient~(SVRG)~\cite{IEEEhowto:kopka6}, 
 both aimed at improving the convergence rate of SGD. Langford et al.~\cite{IEEEhowto:kopka7} and
 McMahan and Streeter studied delay tolerant OGD
 algorithms ~\cite{IEEEhowto:kopka8} where parameter updates are based on
 stale gradient information. 
 
Cyclic block coordinate descent algorithms were introduced by Luo and Tseng~\cite{IEEEhowto:kopka9,IEEEhowto:kopka10}.
Nesterov proposed randomized block coordinate descent~(RBCD)~\cite{IEEEhowto:kopka1} algorithms
for large scale optimization problems. Fercoq and Richtarik~\cite{IEEEhowto:kopka11}
and Singh et al.~\cite{IEEEhowto:kopka12} proposed accelerated randomized block coordinate 
algorithms. More recently, Allen-zhu et al.~\cite{IEEEhowto:kopka13} proposed
faster accelerated coordinate descent methods in which sampling frequencies 
depend on coordinate wise smoothness parameters~(i.e., Lipschitz parameters 
of the corresponding partial derivatives).

Dang and Lan proposed stochastic block mirror descent~(SBMD)~\cite{IEEEhowto:kopka14} which combines
SGD and RBCD. Similar algorithms were proposed by Wang and Banerjee~\cite{IEEEhowto:kopka15},
Hua et al.~\cite{IEEEhowto:kopka16}, Zhao et al.~\cite{IEEEhowto:kopka17} and Zhang and Gu~\cite{IEEEhowto:kopka18} also, who called their algorithms ORBCD, 
R-BCG, MRBCD and ASBCD, respectively. ORBCD and MRBCD were enhanced by variance 
reduction techniques to attain linear rate of convergence for strongly convex
loss functions. On the other hand, ASBCD used 
optimal sampling of training data samples to achieve the same.

More recently, there has been interest in machine learning settings in which training 
data and features are distributed across nodes of a computing cluster, 
or more generally, of a network. Nathan and Klabjan~\cite{IEEEhowto:kopka19} proposed
an algorithm where nodes parallelly update~(possibly overlapping) blocks
of feature parameters based on locally available data samples; this was seen 
as a combination of SVRG and block coordinate descent. Konecny et al.~\cite{IEEEhowto:kopka20}
studied algorithms for nodes connected through a network; this scenario 
was referred to as federated learning. 

\subsection{Our Contribution}
We have proposed two accelerated randomized coordinate descent algorithms
for stochastic optimization and online learning, which we refer to as
SARCD and OARCD, respectively. Expectedly, these algorithms have significantly
less per-iteration computation complexity than accelerated gradient descent
algorithms, e.g., SAGE. Moreover, the proposed algorithms have the following 
properties.
\begin{enumerate}
\item SARCD for general convex objective functions exhibits convergence rate $\O(\frac{n}{\sqrt{T}})$ which matches the best known rates.
\item SARCD for strongly convex objective functions exhibits convergence rate $\O(\frac{n}{T})$ which is strictly
better than $\O(\frac{n\log(T)}{T})$ rate of ORBCD.
\item OARCD for general convex loss functions yields regret bound $\O(\sqrt{nT})$ which is strictly better than 
$\O(n\sqrt{T})$, the regret bound of ORBCD and R-BCG.
\item OARCD for strongly convex loss functions yields regret bound $\O(n\log(T))$ which is strictly better than
$\O(n^2\log(T))$, the regret bound of ORBCD. 
\end{enumerate}
The proposed algorithms can easily be generalized to accelerated randomized 
block coordinate descent algorithms.

\section{The Learning Problems}
We consider machine learning models characterized by input-output pairs,
model parameters~(or, feature parameters) and a loss function. The model parameters are used
to estimate or predict outputs~(also called {\it labels}) from the corresponding inputs~(also called {\it features}).
The loss function is used to measure discrepancy between the predicted
and the actual outputs. To illustrate, let $\xi = (\xi_i,\xi_o)$ be
an input-output pair and $y \in \R^n$ be a vector of the model parameters.
Then $\xi_i$ and $y$ yield an estimate $\hat{\xi}_o$ of $\xi_o$.
Clearly, the loss function, which provides a measure of discrepancy between
$\hat{\xi}_o$ and $\xi_o$, can be seen as a mapping from the couple $(y,\xi)$
to real numbers; let us denote this function as $l(\cdot,\cdot)$.
For example, considering $l_2$-losses,
\[
l(y,\xi) = \Vert \hat{\xi}_o - \xi_o \Vert^2.
\]

Machine learning aims at identifying the model parameters
that minimize the losses for all input-output pairs. We make this notion
precise in the following two subsections which focus on two different premises.

\subsection{Stochastic Optimization}
\label{sec:Stochastic}

Let us assume that we have a collection of input-output pairs, also called training samples.
An input-output pair may appear more than once in the collection, and different pairs can
have different relative frequencies. We aim at determining the model parameters that
minimize the average loss over all the input-output pairs. Towards this, we let $\xi$
denote a random input-output pair with appropriate distribution
and consider the optimization problem\footnote{The proposed algorithm does not need the distribution of $\xi$.}
\[
\min_y \{f(y) \equiv \E_{\xi}[l(y,\xi)]\}.
\]\label{eqn:stoch}
Let $g(y,\xi) = \nabla_y l(y,\xi)$. We assume that
\begin{enumerate}
\item $f: \R^n \rightarrow \R$ is convex and differentiable. Moreover, we assume that
$\nabla f(y)$ is Lipschitz continuous with parameter $L$,
\item $g(y,\xi)$ is an unbiased estimator of
$\nabla f(y)$, i.e., $\E_{\xi}[g(y,\xi)] = \nabla f(y)$,
\item $f(\cdot)$ is strongly convex with parameter $\mu \geq 0$; $\mu > 0$
yields better iteration complexity.\footnote{We allow $\mu = 0$ 
to accommodate general convex loss functions. Strong convexity warrants $\mu > 0$.}\label{ass:strong}
\end{enumerate}

In Section~\ref{sec:SARCD}, we propose an algorithm, SARCD, to solve the above problem.
We establish convergence rates of SARCD for general convex loss functions~(or, cost functions)
and strongly convex loss functions in Theorems~\ref{thm:SARCD-general-conv} and~\ref{thm:SARCD-strongly-conv}, respectively.

\remove{
In high-dimensional optimization problems as above, we often add a regularizer, say
$\psi:\R^n \rightarrow \R$, to the objective function either for well-posedness
or for sparsity. Typical examples of regularizer are $l_1$ and $l_2$ regularizer.
We now encounter the following regularized optimization problem
\[
\min_y f(y) + \psi(y).
\]
We discuss in Section how we can use SARCD to solve this problem.
}
\subsection{Online Learning}
\label{sec:Online}

Here the input-output pairs arrive sequentially in steps and the model parameters
are updated after each step. More precisely, we start with an arbitrary
modelling parameter vector $y_1$. Assuming that we have
model parameters $y_t$ on arrival of the $t$th input-output pair $\xi_t$, we
incur a loss $f_t(y_t) \equiv l(y_t,\xi_t)$\label{eqn:online} and update $y_t$
to $y_{t+1}$ based on $f_t(\cdot)$. For a given $T > 0$, we aim
at generating a sequence of model parameters $y_1,y_2,\dots,y_T$ that minimize the
$T$-step ``regret'' $R(T)$, defined as
\[
R(T) = \sum_{t = 1}^T f_t(y_t) - \min_y \sum_{t = 1}^T f_t(y).
\]
Here we assume that $f_t: \R^n \rightarrow \R$ are convex and differentiable for all $t \geq 1$.
Moreover, we also assume that, for all $t \geq 1$, $\nabla f_t(\cdot)$ are Lipschitz continuous with parameter $L$ and $f_t(\cdot)$ are strongly convex with parameter $\mu \geq 0$.

In Section~\ref{sec:OARCD}, we propose an algorithm, OARCD, to solve the above learning problem.
We provide regret bounds of OARCD for general convex loss functions
and strongly convex loss functions in Theorems~\ref{thm:OARCD-general-conv} and~\ref{thm:OARCD-strongly-conv}, respectively.

\section{Stochastic Accelerated Randomized Coordinate Descent}
\label{sec:SARCD}

SARCD is the ``coordinate descent'' version of SAGE and an ``accelerated''
version of ORBCD. In other words, it is an iterative algorithm in which, at
each iteration, we randomly choose an input-output pair and then a coordinate,
and update only this coordinate of the vector of model parameters.
As is typical of accelerated gradient methods, we update two other sequences
$\{x_t\}$ and $\{z_t\}$ apart from $\{y_t\}$, and we also maintain
two parameter sequences $\{\alpha_t\}$ and $\{L_t\}$ (see ~\ref{sec:Stochastic},~\cite{IEEEhowto:kopka1,IEEEhowto:kopka4}).
Further, we use two constants $a(n)$ and $b(n)$ which we
later set to achieve best convergence results.
We let $\xi_t$ indicate the random input-output pair chosen at
$t$th iteration; $\{\xi_t\}$ are i.i.d.
We also use a random diagonal matrix $Q_t \in \{0,1\}^{n \times n}$
to indicate the coordinate chosen at $t$th iteration;
each $Q_t$ has only one nonzero entry and $\{Q_t\}$ are i.i.d.
\remove{
We also use a sequence $\{Q_t\}$ of random i.i.d diagonal matrices
to indicate the coordinates chosen at different iterations;
Q_t \in \{0,1\}^{n \times n} with only only nonzero entry which corresponds
to the coordinate chosen at $t$th iteration.}
Formally, our algorithm is as follows.

\begin{algorithm}
\caption{Stochastic Accelerated Randomized Coordinate Descent}
\label{algo:SARCD}
\begin{algorithmic}
\State{Input:} Sequences $\{L_t\}$ and $\{\alpha_t\}$
\State{Initialize:} $y_{-1} = z_{-1} = 0$
\For{ $t = 0$ to $T$}
    \State $x_t = (1-\alpha_t)y_{t-1} + \alpha_t z_{t-1}$
    \State $y_{t} = \arg\min_{x} \{\langle a(n)Q_t g(x_t,\xi_t), x - x_{t}\rangle + \frac{L_t}{2} \Vert x-x_t\Vert^2\}$
    \State $z_{t} = z_{t-1} - \frac{a(n)b^2(n)}{nL_t \alpha_t + \mu a(n)b(n)} [\frac{nL_{t}}{a(n)b(n)}(x_t-y_t) + \frac{\mu}{b(n)}(z_{t-1} - x_t)]$
\EndFor\\
Output: $y_T$
\end{algorithmic}
\end{algorithm}

Clearly, $\xi_{t}$ and $Q_{t}$ are independent and are also independent of $x_t$.
Let $\Delta_t = a(n)Q_t(g(x_t,\xi_t) - \nabla f(x_t))$.
Let $\delta_t = L_t(x_t-y_t) = a(n)Q_t g(x_t,\xi_t)$ be the gradient mapping involved in updating $y_t$. Also, let $\H_t$ denote
the history of the algorithm until time $t$. More explicitly,
\[
\H_t = (\xi_0,Q_0,\xi_1,Q_1,\dots ,\xi_{t-1},Q_{t-1}).
\]
Notice that $(x_l,y_l,z_l,l=0,\dots,t-1)$ and $x_t$ are functions of $\H_t$.
We first establish the following lower bound on $f(x)$.

\begin{lemma}
\label{lemma:SARCD}
For $t \geq 0$,
\begin{align*}
f(x) \geq & \ \E[f(y_t)|\H_t] + \frac{n}{a(n)}\E[\langle\delta_t,x-x_{t}\rangle|\H_t] + \frac{n}{a(n)}\E[\langle\Delta_t,y_t-x\rangle|\H_t] \\
 & + \frac{\frac{2}{a(n)}L_t-L}{2L_t^2}\E[\Vert \delta\Vert^2 |\H_t] + \frac{(n-1)}{a(n)L_t}\E[\langle\Delta_t,\delta_t\rangle|\H_t] + \frac{\mu}{2} \Vert x-x_t\Vert^2.
\end{align*}
\end{lemma}

\begin{proof}
Using strong convexity of $f$ (see assumption ~\ref{ass:strong},~\cite{IEEEhowto:kopka1})
\begin{align}
f(x) & \geq f(x_t) + \langle \nabla f(x_t),x-x_t\rangle + \frac{\mu}{2}\Vert x-x_t\Vert^2 \nonumber \\
   & = f(x_t) + n\E[\langle Q_t\nabla f(x_t), x-x_t \rangle |\H_t] + \frac{\mu}{2}\Vert x-x_{t}\Vert^2. \label{eqn:strong-convexity-stoch}
\end{align}
On the other hand, by Descent Lemma~\cite{IEEEhowto:kopka1},
\begin{equation}
f(y_t) \leq f(x_t) + \langle Q_t \nabla f(x_t), y_t-x_t \rangle + \frac{L}{2} \Vert x_t-y_t \Vert^2. \label{eqn:descent-lemma-stoch}
\end{equation}
Taking expectation in~\eqref{eqn:descent-lemma-stoch}~(conditioned on $\H_t$) and then combining with~\eqref{eqn:strong-convexity-stoch},
\begin{align*}
f(x) \geq & \E[f(y_t) | \H_t] - \E[\langle Q_t \nabla f(x_t), y_t-x_t\rangle |\H_t] - \frac{L}{2}\E[\Vert y_{t}-x_{t}\Vert^2 | \H_t] \\
 & + n \E[\langle Q_t \nabla f(x_{t}), x-x_{t}\rangle | \H_t] + \frac{\mu}{2}\Vert x-x_{t}\Vert^{2}\\
 = & \E[f(y_t) | \H_t] - n\E[\langle Q_t \nabla f(x_t) - Q_t g(x_t,\xi_t), y_t-x_t\rangle | \H_t]  \\
 & + (n-1) \E[\langle Q_t \nabla f(x_t) - Q_t g(x_t,\xi_t),y_t-x_t\rangle | \H_t] \\
  &  - \E[\langle Q_t g(x_t,\xi_t),y_t-x_t\rangle| \H_t]+ n\E[\langle Q_t \nabla f(x_t) - Q_t g(x_t,\xi_t),x-x_t\rangle | \H_t]\\
  & + n\E[\langle Q_t g(x_t,\xi_t),x-x_t\rangle | \H_t] - \frac{L}{2}\E[\Vert y_t-x_t\Vert^2| \H_t] + \frac{\mu}{2}\Vert x-x_t\Vert^2 \\
 = & \E[f(y_t) | \H_t] - \frac{n}{a(n)}\E[\langle\Delta_t,x-y_t\rangle | \H_t] + \frac{n}{a(n)}\E[\langle\delta_t,x-x_t\rangle | \H_t] \\
 & + \frac{\frac{2}{a(n)}L_t-L}{2L_t^2} \E[\Vert\delta \Vert^2 | \H_t] + \frac{(n-1)}{a(n)L_t}E[\langle\Delta_t,\delta_t\rangle|\H_t] + \frac{\mu}{2}\Vert x-x_{t}\Vert^{2},
\end{align*}
which is the desired inequality.
\qed
\end{proof}

\begin{proposition}
\label{prop:SARCD}
Assume that $\Vert g(x_t,\xi_t) - \nabla f(x_t)\Vert \leq \sigma$ and $L_t > a(n)L$ for all $t \geq 0$. Then, for all $t \geq 0$,
\begin{align*}
\MoveEqLeft \E[f(y_t)-f(x)|\H_t] \\
\leq & (1-\alpha_t)(f(y_{t-1})-f(x)) + \frac{\sigma^2}{2n(\frac{L_t}{a(n)} -L)} +
\frac{n\alpha_t}{a(n)}\E[\langle\Delta_{t},x-z_{t-1}\rangle|\H_t] \\
& + \frac{n L_t \alpha_t^2}{2a(n)}\E[\Vert x-z_{t-1}\Vert^2-\Vert x-z_t\Vert^2|\H_t] - \frac{\mu \alpha_t}{2}\E[\Vert x-z_t\Vert^2|\H_t]
\end{align*}
\end{proposition}
\begin{proof}
Let us first notice that~(see Algorithm~\ref{algo:SARCD})
\begin{equation*}
z_t = \arg\min_x \left(\langle b(n) \delta_{t},x-x_{t} \rangle +\frac{L_{t}\alpha_{t}}{2}\Vert x-z_{t-1}\Vert^{2} + \frac{\mu a(n)b(n)}{2n} \Vert x-x_{t}\Vert^{2}\right),
\end{equation*}
and also that the objective function in this minimization problem is strongly convex with parameter
$(L_t\alpha_t+\frac{\mu a(n)b(n)}{n})$. Hence,
\begin{align*}
\MoveEqLeft \langle b(n)\delta_{t},x-x_{t}\rangle +\frac{L_{t}\alpha_{t}}{2}\Vert x-z_{t-1}\Vert^{2} + \frac{\mu a(n)b(n)}{2n} \Vert x-x_{t}\Vert^{2} \\
\geq & \langle b(n)\delta_{t},z_t-x_{t}\rangle +\frac{L_{t}\alpha_{t}}{2}\Vert z_t-z_{t-1}\Vert^{2} + \frac{\mu a(n)b(n)}{2n} \Vert z_t-x_{t}\Vert^{2} \\
& + \frac{L_{t}\alpha_{t}}{2}\Vert x-z_{t}\Vert^{2} + \frac{\mu a(n)b(n)}{2n} \Vert x-z_{t}\Vert^{2}.
\end{align*}
Using this in Lemma~\ref{lemma:SARCD},
\begin{align}
f(x) \geq & \E[f(y_{t}) | \H_t]  + \frac{n}{a(n)}\E[\langle \delta_{t},z_{t}-x_{t}\rangle  | \H_t] - \frac{n}{a(n)}\E[\Delta_{t},x-y_{t}|\H_t] \nonumber \\
 & + \frac{nL_{t}\alpha_{t}}{2a(n)b(n)}\E[\Vert z_{t}-z_{t-1}\Vert^{2} + \Vert x-z_{t}\Vert^{2} - \Vert x-z_{t-1}\Vert^{2} | \H_t] \nonumber \\
 & + \frac{\frac{2}{a(n)}L_{t}-L}{2L_{t}^{2}}\E[\Vert\delta_{t}\Vert^{2} | \H_t] + \frac{n-1}{a(n)L_{t}}\E[\langle \Delta_{t},\delta_{t}\rangle | \H_t] + \frac{\mu}{2}\E[\Vert x-z_{t}\Vert^{2}|\H_t] \label{eqn:fx-bound}
\end{align}
where we have dropped the term $\frac{\mu}{2} \Vert z_t-x_t\Vert^{2}$ from the right hand side without affecting the inequality. Also, substituting $x = y_{t-1}$ in Lemma~\ref{lemma:SARCD},
\begin{align}
f(y_{t-1}) \geq & \E[f(y_t)|\H_t] - \frac{n}{a(n)}\E[\langle\Delta_t,y_{t-1}-y_t\rangle|\H_t] + + \frac{n}{a(n)}\E[\langle\delta_t,y_{t-1}-x_t\rangle|\H_t]\nonumber \\
 & + \frac{\frac{2}{a(n)}L_t-L}{2L_t^2}\E[\Vert \delta\Vert^2 |\H_t] + \frac{(n-1)}{a(n)L_t}\E[\langle\Delta_t,\delta_t\rangle|\H_t] \label{eqn:fyt-1-bound},
\end{align}
where again we have dropped the term $\frac{\mu}{2} \Vert y_{t-1}-x_{t}\Vert^{2}$ from the right hand side. Now, multiplying~\eqref{eqn:fx-bound} by $\alpha_t$ and~\eqref{eqn:fyt-1-bound} by $(1-\alpha_t)$ and adding,
\begin{align*}
\MoveEqLeft \alpha_tf(x) + (1-\alpha_t)f(y_{t-1}) \\
\geq & \E[f(y_{t})|\H_t] + \frac{n\alpha_{t}}{a(n)}\E[\langle \delta_{t},z_{t}-x_{t}\rangle | \H_t] - \frac{n\alpha_t}{a(n)}\E[\Delta_{t},x-y_{t}|\H_t] \\
& + \frac{nL_t\alpha^2_t}{2a(n)b(n)}\E[\Vert z_{t}-z_{t-1}\Vert^{2} + \Vert x-z_{t}\Vert^{2} - \Vert x-z_{t-1}\Vert^{2} | \H_t] \\
& - \frac{n(1-\alpha_t)}{a(n)}\E[\langle \Delta_{t},y_{t-1}-y_{t}\rangle |\H_t] + \frac{n(1-\alpha_t)}{a(n)}\E[\langle \delta_t,y_{t-1}-x_{t}\rangle  | \H_t] \\
&  + \frac{\frac{2}{a(n)}L_{t}-L}{2L_{t}^{2}}\E[\Vert\delta\Vert^2 | \H_t] + \frac{(n-1)}{a(n)L_t}\E[\langle \Delta_{t},\delta_{t}\rangle | \H_t] + \frac{\alpha_t \mu}{2}\E[\Vert x-z_t\Vert^2|\H_t].
\end{align*}
Rearranging the terms,
\begin{align*}
\MoveEqLeft \E[f(y_{t})-f(x)|\H_t] \\
\leq & (1-\alpha_{t})(f(y_{t-1})-f(x)) - \frac{\frac{2}{a(n)}L_{t}-L}{2L_{t}^{2}}\E[\Vert\delta\Vert^2 | \H_t]
- \frac{(n-1)}{a(n)L_t}\E[\langle \Delta_{t},\delta_{t}\rangle|\H_t] \\
 & - \frac{nL_t \alpha^2_t}{2a(n)b(n)}\E[\Vert z_{t}-z_{t-1}\Vert^{2}|\H_t] + A + B \\
 & + \frac{n L_t \alpha_t^2}{2a(n)b(n)}\E[\Vert x-z_{t-1}\Vert^{2}-\Vert x-z_{t}\Vert^{2}|\H_t] - \frac{\mu \alpha_{t}}{2} \E[\Vert x-z_{t}\Vert^2|\H_t],
\end{align*}
where
\begin{align*}
A &= -\frac{n}{a(n)}\E[\langle \delta_{t},\alpha_{t}(z_{t}-x_{t})+(1-\alpha_{t})(y_{t-1}-x_{t})\rangle |\H_t]\\
\text{and } B &= \frac{n}{a(n)}\E[\langle \Delta_{t}, \alpha_t(x-y_t)+ (1-\alpha_t) (y_{t-1}-y_{t})\rangle |\H_t].
\end{align*}
We can see that
\begin{align*}
A  &= -\frac{n}{a(n)}\E[\langle \delta_{t},\alpha_t(z_t-z_{t-1})\rangle + \langle\delta_t,\alpha_t(z_{t-1}-x_t)+(1-\alpha_{t})(y_{t-1}-x_t)\rangle|\H_t] \\
& = \frac{n}{a(n)}\E[\langle \delta_{t},\alpha_t(z_{t-1}-z_t)\rangle|\H_t] \\
& \leq \frac{n}{a(n)}\E\left[\frac{\Vert\delta_{t}\Vert^2}{2nL_t} + \frac{n L_t\alpha_t^2}{2}\Vert z_t-z_{t-1}\Vert^2 \Big| \H_t\right],
\end{align*}
where we use the update rule of $x_t$ to get the second equality~(see Algorithm~\ref{algo:SARCD}) and then the Young's inequality.\footnote{The Young's inequality states that $\langle x,y\rangle \leq \frac{\Vert x \Vert^2}{2a} + \frac{a \Vert y \Vert^2}{2}$ for any $a > 0$.}
Further,
\begin{align*}
B &=  \frac{n}{a(n)}\E[\langle \Delta_{t}, \alpha_t x+(1-\alpha_t)y_{t-1} - x_t \rangle + \langle \Delta_t, x_t - y_t \rangle |\H_t] \\
  &= \frac{n}{a(n)}\E\left[\alpha_t \langle \Delta_t,x-z_{t-1} \rangle + \frac{\langle \Delta_t,\delta_t\rangle}{L_t} \Big | \H_t \right],
\end{align*}
where we again use the update rule of $x_t$ in the last equality.
Using the above bound on $A$, the expression for $B$ and Cauchy-Schwartz inequality~(to infer $\langle \Delta,\delta \rangle \leq \Vert \Delta_t \Vert \Vert \delta_t \Vert$), and setting $b(n) = \frac{1}{n}$,
\begin{align*}
\E[f(y_t)-f(x)|\H_t] \leq & (1-\alpha_t)(f(y_{t-1})-f(x)) - \frac{\frac{1}{a(n)}L_t - L}{2L_t^2} \E[ \Vert\delta\Vert^2 |\H_t] \\
& + \frac{1}{a(n)L_t}\E[\Vert \Delta_t \Vert \Vert \delta_t \Vert|\H_t] +
\frac{n\alpha_t}{a(n)}\E[\langle\Delta_{t},x-z_{t-1}\rangle|\H_t] \\
& \hspace{-0.33in} + \frac{n^2 L_t \alpha_t^2}{2a(n)}\E[\Vert x-z_{t-1}\Vert^{2}-\Vert x-z_{t}\Vert^{2}|\H_t] - \frac{\mu \alpha_{t}}{2} \E[\Vert x-z_{t}\Vert^2|\H_t].
\end{align*}
We now set $a=\frac{\frac{L_t}{a(n)} - L}{2}, b=\frac{\Vert \Delta_t \Vert}{a(n)}, \theta=\frac{\Vert\delta_t\Vert}{L_t}$,
and use the fact that $-a\theta^{2} + b\theta \leq \frac{b^{2}}{4a} , a, b \geq 0$, to get
\begin{align*}
\E[f(y_t)-f(x)|\H_t] \leq & (1-\alpha_t)(f(y_{t-1})-f(x))\\
& + \frac{1}{2a^2(n)(\frac{L_t}{a(n)} -L)}\E[\Vert \Delta_t \Vert^2 |\H_t] +
\frac{n\alpha_t}{a(n)}\E[\langle\Delta_{t},x-z_{t-1}\rangle|\H_t] \\
& \hspace{-0.33in} + \frac{n^2 L_t \alpha_t^2}{2a(n)}\E[\Vert x-z_{t-1}\Vert^{2}-\Vert x-z_{t}\Vert^{2}|\H_t] - \frac{\mu \alpha_{t}}{2} \E[\Vert x-z_{t}\Vert^2|\H_t].
\end{align*}
Finally, we get the desired result by using the bound $\E[\Vert \Delta_t \Vert^2 |\H_t] \leq \frac{a^2(n)\sigma^2}{n}$.
\qed
\end{proof}

Let $x^{\ast}$ be the optimal solution of the optimization problem~\eqref{eqn:stoch}. Then
\begin{align*}
\E_{\H_t,\xi_t,Q_t}[\langle \Delta_t,x^{\ast}-z_{t-1}\rangle|\H_t] & = E_{\H_t}[\E_{\xi_t,Q_t}\langle \Delta_t,x^{\ast}-z_{t-1}\rangle|\H_t] \\
      &= E_{\H_t}[\langle \E_{\xi_t,Q_t}[\Delta_t|\H_t],x^{\ast}-z_{t-1}\rangle] = 0
\end{align*}
because $\E_{\xi_t,Q_t}[\Delta_t|\H_t] = 0$ owing to unbiasedness of $g(y,\xi)$.
We set $x = x^\ast$ in Proposition~\ref{prop:SARCD} and take expectation on both the sides to
obtain the following corollary.
\begin{corollary}
\label{cor:SARCD}
Assume that $\Vert g(x_t,\xi_t) - \nabla f(x_t)\Vert \leq \sigma$ and $L_t > a(n)L$ for all $t \geq 0$. Then, for all $t \geq 0$,
\begin{align*}
\E[f(y_t)-f(x^\ast)] \leq & (1-\alpha_t)\E[f(y_{t-1})-f(x^\ast)] + \frac{\sigma^2}{2n(\frac{L_t}{a(n)} -L)} \\
& \hspace{-0.33in} + \frac{n^2 L_t \alpha_t^2}{2a(n)}\E[\Vert x^\ast-z_{t-1}\Vert^2-\Vert x^\ast-z_t\Vert^2] - \frac{\mu \alpha_t}{2}\E[\Vert x^\ast-z_t\Vert^2].
\end{align*}
\end{corollary}

We now appropriately set $\alpha_t, L_t$ and $a(n)$ to obtain rapid convergence
of $\E[f(y_T)]$ to $f(x^\ast)$. We first consider the case when $\mu = 0$, i.e., $f$ is not
strongly convex.
\begin{theorem}
\label{thm:SARCD-general-conv}
Assume that that $\mu =0$, $\Vert g(x_t,\xi_t) - \nabla f(x_t)\Vert \leq \sigma$ and $\E[\Vert x^\ast-z_t \Vert^2] \leq D^{2}$. Set $a(n) = n, \alpha_t=\frac{2}{t+2}$ and $L_t= b(t+1)^{\beta} + a(n)L$ for $t \geq 0$, where
$\beta = \frac{3}{2}$ and $b > 0$ is a constant. Then
\[
\E[f(y_T)-f(x^\ast)] \leq \frac{2n^2D^2L}{T^2} + \left(2nD^2 b+\frac{4\sigma^2}{3b}\right)\frac{1}{\sqrt{T}}.
\]
\end{theorem}
\begin{proof}
Dividing both sides in Corollary~\ref{cor:SARCD} by $\alpha_t^2$,
\begin{align*}
\frac{1}{\alpha^2_t}\E[f(y_t)-f(x^\ast)] \leq & \frac{(1-\alpha_t)}{\alpha^2_t}\E[f(y_{t-1})-f(x^\ast)] \\
& + \frac{n^2 L_t}{2a(n)}\E[\Vert x^\ast-z_{t-1}\Vert^2-\Vert x^\ast-z_t\Vert^2] + \frac{\sigma^2}{2n\alpha_t^2(\frac{L_t}{a(n)} -L)}\\
 \leq & \frac{1}{\alpha^2_{t-1}}\E[f(y_{t-1})-f(x^\ast)] \\
& + \frac{n^2 L_t}{2a(n)}\E[\Vert x^\ast-z_{t-1}\Vert^2-\Vert x^\ast-z_t\Vert^2] + \frac{\sigma^2}{2n\alpha_t^2(\frac{L_t}{a(n)} -L)},
\end{align*}
where the last inequality holds because
\[
\frac{1-\alpha_t}{\alpha_t^2} = \frac{\frac{t}{t+2}}{\frac{4}{(t+2)^2}} = \frac{t(t+2)}{4} \leq \frac{(t+1)^2}{4} = \frac{1}{\alpha_{t-1}^2}.
\]
Similar inequalities hold for all $t = 1,\dots,T$, which, when put together, yield
\begin{align*}
\frac{1}{\alpha^2_T}\E[f(y_T)-f(x^\ast)] \leq & \frac{1}{\alpha_0^2}\E[f(y_0)-f(x^\ast)] + \frac{\sigma^2}{2n}\sum_{t=1}^{T}\frac{1}{\alpha_t^2(\frac{L_t}{a(n)} - L)} \\
& + \sum_{t=1}^{T}\frac{L_t n^2}{2a(n)}\E[\Vert x^\ast-z_{t-1}\Vert^2] - \sum_{t=1}^T \frac{L_t n^2}{2a(n)}\E[\Vert x^\ast-z_{t}\Vert^2] \\
\leq & \frac{\sigma^2}{2n}\sum_{t=0}^{T}\frac{1}{\alpha_t^2(\frac{L_t}{a(n)} - L)} \\
& + \sum_{t=0}^{T}\frac{L_t n^2}{2a(n)}\E[\Vert x^\ast-z_{t-1}\Vert^2] - \sum_{t=0}^T \frac{L_t n^2}{2a(n)}\E[\Vert x^\ast-z_{t}\Vert^2] \\
= & \frac{\sigma^2}{2n}\sum_{t=0}^{T}\frac{1}{\alpha_t^2(\frac{L_t}{a(n)} - L)} + \frac{n^2}{2a(n)}\sum_{t=0}^{T}\E[\Vert x^\ast-z_{t}\Vert^2](L_{t+1} - L_t) \\
& + \frac{L_0 n^2}{2a(n)}\Vert x^\ast-z_{-1}\Vert^2 - \frac{L_{T+1} n^2}{2a(n)}\E[\Vert x^\ast-z_T\Vert^2] \\
\leq & \frac{\sigma^2}{2n}\sum_{t=0}^T \frac{(t+2)^2}{4(\frac{L_{t}}{a(n)} - L)} + \frac{n^2D^2}{2a(n)}(L_{T+1} - L_0) + \frac{L_0 n^2 D^2}{2a(n)},
\end{align*}
where the second inequality is obtained by using Corollary~\ref{cor:SARCD} for $t=0$ and the last inequality by substituting the value of $\alpha_t$. We thus get
\begin{align*}
\MoveEqLeft \E[f(y_T)-f(x^\ast)] \\
 & \leq \frac{\sigma^2}{2n(T+2)^2}\sum_{t=0}^T\frac{(t+2)^2}{(\frac{L_t}{a(n)} - L)} + \frac{2n^2 D^2}{(T+2)^2 a(n)}(L_{T+1}-L_0) + \frac{2n^2 D^2 L_0}{(T+2)^2 a(n)} \\
 & \leq \frac{2\sigma^2}{n(T+2)^2}\sum_{t=0}^T\frac{(t+1)^2 a(n)}{b(t+1)^\beta} + \frac{2n^2 D^2}{(T+2)^2 a(n)}(b(T+2)^\beta + a(n)L) \\
 & \leq \frac{2\sigma^2 a(n)}{n(T+2)^2 b}\frac{(T+2)^{3-\beta}}{3-\beta} + \frac{2n^2 D^2 }{(T+2)^2 a(n)}(b(T+1)^{\beta} + a(n)L) \\
 & \leq \frac{2\sigma^2 a(n)(T+2)^{1-\beta}}{nb(3-\beta)} + \frac{2n^{2}D^{2}b(T+2)^{\beta-2}}{a(n)} + \frac{2n^{2}D^{2}L}{(T+2)^{2}},
\end{align*}
where the second inequality is obtained by using $(t+2)^2 \leq 4(t+1)^2$ and also substituting the value
of $L_t$. Finally, we get the desired bound by setting $a(n) = n$ and $\beta = \frac{3}{2}$.
\qed
\end{proof}

We now consider the case when $\mu > 0$, i.e., $f$ is strongly convex.
\begin{theorem}
\label{thm:SARCD-strongly-conv}
Assume that that $\mu > 0$,$\Vert g(x_t,\xi_t) - \nabla f(x_t)\Vert \leq \sigma$ and $\E[\Vert x^\ast-z_t \Vert^2] \leq D^{2}$ for some $D > 0$. Set
$\alpha_0 = 1, L_0 = a(n)L +\frac{a(n)\mu}{n^2}$ and $\alpha_t= \sqrt{\lambda_{t-1}+\frac{\lambda_{t-1}^{2}}{4}} - \frac{\lambda_{t-1}}{2}, L_t= a(n)L + \frac{a(n)\mu} {n^2\lambda_{t-1}}$ for $t \geq 1$ where $\lambda_{0} = 1$ and $\lambda_{t} = \prod_{k=1}^{t}(1-\alpha_k)$ for $t \geq 1$. Then
\[
\E[f(y_T)-f(x^\ast)] \leq \frac{2(n^2 L+\mu)D^2 }{(T+2)^2} + \frac{2n\sigma^2}{(T+2)\mu}\left(1+  \frac{2\ln(T+1)}{T+2}\right).
\]
\end{theorem}
\begin{proof}
We first make a few observations.
\begin{enumerate}[(a)]
\item Using definitions of $\alpha_t$ and $\lambda_t$,
\begin{equation*}
\alpha_t^2 = (1 - \alpha_t)\lambda_{t-1} = \lambda_t, \text{for all} \ t \geq 1.
\end{equation*}
\item Since $\alpha_0^2 = \lambda_0 = 1$, we can also conclude
\[
\frac{1 - \alpha_t}{\alpha^2_t} = \frac{1}{\alpha^2_{t-1}}, \text{for all} \ t \geq 1.
\]
\item For all $t \geq 1$,
\[
L_{t+1} - L_t = \frac{a(n)\mu}{n^2\lambda_t} - \frac{a(n)\mu (1 - \alpha_t)}{n^2\lambda_t} = \frac{a(n)\mu \alpha_t}{n^2\lambda_t} = \frac{a(n)\mu}{n^2\alpha_t}.
\]
\item Finally, $\alpha_t \leq \frac{2}{t+2}$ for all $t \geq 1$. We can show this via induction. Notice that $\alpha_1 = \frac{\sqrt{5}-1}{2} \leq \frac{2}{3}$. Now assume that, for some $t \geq 2$, $\alpha_{t-1} \leq \frac{2}{t+1}$ and
    $\alpha_{t} > \frac{2}{t+2}$. But the latter implies
    \[
    \frac{1}{\alpha^2_{t-1}} = \frac{1-\alpha_t}{\alpha^2_t} < \frac{1 - \frac{2}{t+2}}{\frac{4}{(t+2)^2}} = \frac{(t+2)^2 -2(t+2)}{4} \leq \left(\frac{t+1}{2}\right)^2,
    \]
    which contradicts the former. This completes the argument.
\end{enumerate}

Using~(a),~(b),~(c) above and Corollary~\ref{cor:SARCD} as in the proof of Theorem~\ref{thm:SARCD-general-conv},
\begin{align*}
\MoveEqLeft \frac{\E[f(y_T)-f(x^{\ast})]}{\alpha_T^2} + \frac{n^2 L_{T+1}}{2a(n)}\E[\norm{x^{\ast} - z_T}^2] \\
\leq & \frac{\E[f(y_0) - f(x^{\ast})]}{\alpha_{0}^{2}} + \frac{n^2 L_{1}}{2a(n)}E[\norm{x^{\ast}-z_{0}}^{2}] + \frac{\sigma^{2}}{2n}\sum_{t=1}^{T} \frac{1}{\alpha_{t}^{2}(\frac{L_{t}}{a(n)}-L)}.
\end{align*}
Using Corollary~\ref{cor:SARCD} once more for $t=0$ and $L_1 = L_0 = a(n)(L+\frac{\mu}{n^2})$,
\begin{align*}
\MoveEqLeft \frac{\E[f(y_T)-f(x^{\ast})]}{\alpha_T^2} + \frac{n^2 L_{T+1}}{2a(n)}\E[\norm{x^{\ast} - z_T}^2] \\
 \leq & \frac{n^2 L+ \mu}{2}\E[\norm{x^{\ast}-z_{-1}}^{2}] + \frac{\sigma^{2}}{2n}\sum_{t=0}^{T}
 \frac{1}{\alpha_{t}^{2}(\frac{L_{t}}{a(n)}-L)}.
\end{align*}
Therefore,
\begin{align*}
\E[f(y_T)-f(x^{\ast})] & \leq \frac{\alpha_T^2 (n^2 L+ \mu)D^2 }{2} + \frac{\alpha_T^2 \sigma^{2}}{2n}\sum_{t=0}^{T} \frac{1}{\alpha_{t}^{2}(\frac{L_{t}}{a(n)}-L)} \\
& \leq \frac{\alpha_T^2 (n^2 L+ \mu)D^2 }{2} + \frac{n\alpha_T^2 \sigma^{2}}{2\mu}\left(1+ \sum_{t=1}^{T}\frac{\lambda_{t-1}}{\alpha_t^2}\right) \\
& \leq \frac{\alpha_T^2 (n^2 L+ \mu)D^2 }{2} + \frac{n\alpha_T^2 \sigma^{2}}{2\mu}\left(1  +\sum_{t=1}^{T}\frac{1}{1-\alpha_t}\right) \\
& \leq \frac{\alpha_T^2 (n^2 L+ \mu)D^2 }{2} + \frac{n\alpha_T^2 \sigma^{2}}{2\mu}\left(1  +\sum_{t=1}^{T}\frac{t+2}{t}\right) \\
& \leq \frac{2(n^2 L+\mu)D^2 }{(T+2)^2} + \frac{2n\sigma^2}{(T+2)\mu} + \frac{4n\sigma^2\ln(T+1)}{(T+2)^2 \mu},
\end{align*}
where the second inequality follows from our setting of $L_t$, the third from~(b)
and the fourth from~(d).
\qed
\remove{
\begin{equation*}
\leq \frac{1-\alpha_{1}}{\alpha_{1}^{2}} E[f(y_{1})-f(x^{\ast})] + \frac{\sigma^{2}}{2}\sum_{t=1}^{T} \frac{1}{\alpha_{t}^{2}(L_{t}-L)}
\end{equation*}

\begin{equation*}
E[f(y_{t})-f(x^{\ast})] \leq \frac{\alpha_{T}^{2}(L+\mu)D^{2}}{2} + \frac{\alpha_{T}^{2}\sigma^{2}}{2 \mu}\sum_{t=1}^{T}\frac{\lambda_{t-1}}{\alpha_{t}^{2}} - \frac{\mu \alpha_{t}}{2}\norm{x^{\ast} - z_{t}}^{2}
\end{equation*}
}
\end{proof}
\begin{remark}
\begin{enumerate}
\item Settings of $L_t$ and $\alpha_t$ in Theorems~\ref{thm:SARCD-general-conv} and~\ref{thm:SARCD-strongly-conv} do not require knowledge of $\sigma$ and the number of iterations $T$.
\item The convergence rate bound in the case of strongly convex objective functions is independent of the choice of $a(n)$.
    \end{enumerate}
\end{remark}

\section{Online Accelerated Randomized Coordinate Descent}
\label{sec:OARCD}

OARCD can also be seen as the ``coordinate descent'' version of the SAGE-based
online learning algorithm and an ``accelerated''
version of ORBCD for online learning. Here, at
each step $t$, we encounter an input-output pair and incur a loss $f_t(y_t)$. We then randomly
choose a coordinate of $y_t$,
and update only this coordinate based on $f_t(\cdot)$.
As in the case of SARCD, we maintain two other sequences,
$\{x_t\}$ and $\{z_t\}$, and two parameter sequences, $\{\alpha_t\}$ and $\{L_t\}$,
and also use two constants $a(n)$ and $b(n)$ to achieve
optimal regret bounds.
We again use a random diagonal matrix $Q_t \in \{0,1\}^{n \times n}$
to indicate the coordinate chosen at $t$th step.
\remove{
We also use a sequence $\{Q_t\}$ of random i.i.d diagonal matrices
to indicate the coordinates chosen at different iterations;
Q_t \in \{0,1\}^{n \times n} with only only nonzero entry which corresponds
to the coordinate chosen at $t$th iteration.}
Formally, this algorithm is as follows.

\begin{algorithm}
\caption{Online Accelerated Randomized Coordinate Descent}
\label{algo:OARCD}
\begin{algorithmic}
\State{Input:} Sequences $\{L_t\}$ and $\{\alpha_t\}$
\State{Initialize:} $y_{0} = z_{0} = 0$
\For{ $t = 1,2,\dots$,}
    \State $x_t = (1-\alpha_t)y_{t-1} + \alpha_t z_{t-1}$
    \State $y_{t} = \arg\min_{x} \{\langle a(n)Q_t \nabla f_t(y_t), x - x_{t}\rangle + \frac{L_t}{2} \Vert x-x_t\Vert^2\}$
    \State $z_{t} = z_{t-1} - \frac{a(n)b^2(n)\alpha_t}{n L_t + a(n)b(n)\alpha_t \mu} [\frac{nL_{t}}{a(n)b(n)}(x_t-y_t) + \frac{\mu}{b(n)}(z_{t-1} - x_t)]$
\EndFor
\end{algorithmic}
\end{algorithm}

Let $\delta_t = L_t(x_t-y_t) = a(n)Q_t \nabla f_{t-1}(x_t)$ and $\H_t$ denote
the history of the algorithm until time $t$.
We first establish the following lower bound on $f_{t-1}(x)$.

\begin{lemma}
\label{lemma:OARCD}
For $t \geq 1$,
\begin{align*}
f_{t-1}(x) \geq & \ \E[f_{t-1}(y_t)|\H_t] + \frac{n}{a(n)}\E[\langle\delta_t,x-x_{t}\rangle|\H_t] \\
 & \ + \frac{\frac{2}{a(n)}L_t-L}{2L_t^2}\E[\Vert \delta\Vert^2 |\H_t] + \frac{\mu}{2} \Vert x-x_t\Vert^2.
\end{align*}
\end{lemma}
\begin{proof}
From strong convexity of $f_{t-1}$,
\begin{align}
f_{t-1}(x) & \geq  f_{t-1}(x_t) + \langle \nabla f_{t-1}(x_t), x-x_t\rangle + \frac{\mu}{2}\Vert x-x_{t}\Vert^2 \nonumber \\
& =  f_{t-1}(x_t) + n \E[\langle Q_{t}\nabla f_{t-1}(x_t), x-x_t\rangle \vert \H_t] + \frac{\mu}{2}\Vert x-x_{t}\Vert^2.
\label{eqn:strong-convexity-online}
\end{align}
On the other hand, by Descent Lemma,
\begin{equation}
f_{t-1}(y_t) \leq f_{t-1}(x_t) + \langle Q_t \nabla f_{t-1}(x_t),y_t-x_t \rangle + \frac{L}{2} \Vert y_t-x_t\Vert^2.
\label{eqn:descent-lemma-online}
\end{equation}
Taking expectation in~\eqref{eqn:descent-lemma-online}~(conditioned on $\H_t$) and then combining with~\eqref{eqn:strong-convexity-online},
\begin{align*}
f_{t-1}(x) \geq & \E[f_{t-1}(y_t) | \H_t] + n \E[\langle Q_t \nabla f_{t-1}(x_{t}), x-x_{t}\rangle | \H_t] \\
& - \E[\langle Q_t \nabla f_{t-1}(x_t), y_t-x_t\rangle \vert \H_t] - \frac{L}{2}\E[\Vert y_{t}-x_{t}\Vert^2 | \H_t] + \frac{\mu}{2}\Vert x-x_{t}\Vert^{2}\\
= & \E[f_{t-1}(y_t)|\H_t] + \frac{n}{a(n)}\E[\langle\delta_t,x-x_{t}\rangle|\H_t] \\
 & + \frac{1}{a(n)L_t}\E[\Vert \delta\Vert^2 |\H_t] - \frac{L}{2L_t^2}\E[\Vert \delta\Vert^2 |\H_t] + \frac{\mu}{2} \Vert x-x_t\Vert^2,
\end{align*}
which is the desired inequality.
\qed
\end{proof}

\begin{proposition}
\label{prop:OARCD}
Assume that $\Vert\nabla f_t(x)\Vert \leq R$ and $L_t > L$ for all $t \geq 1$. Then, for all $t \geq 1$,
\begin{align*}
 \MoveEqLeft \E[f_{t-1}(y_{t-1})] - f_{t-1}(x) \\
  \leq & \frac{a(n)R^{2}}{2n(1-\alpha_t)(L_t-L)} + \frac{nL_t}{2a(n)\alpha_t}\E[\Vert x-z_{t-1}\Vert^2 - \Vert x-z_t\Vert^2]  \\
  & - \frac{nL_t}{2a(n)}\E[\Vert z_t-y_t \Vert^2] +  \frac{n((1-\alpha_t^2)L_t-\alpha_t(1-\alpha_t)L)}{2a(n)}\E[\Vert y_{t-1}-z_{t-1}\Vert^2] \\
  & +\frac{a(n)(2(n-1)L_t+ (a(n)-n)L)R^2}{2nL_t^2} -\frac{\mu}{2}\E[\Vert x-z_t\Vert^2].
 \end{align*}
\end{proposition}
\begin{proof}
From the update equation of $z_t$~(see Algorithm~\ref{algo:OARCD}),
\begin{equation*}
z_t = \arg\min_x \left(\langle b(n) \delta_{t},x-x_{t} \rangle +\frac{L_t}{2\alpha_t}\Vert x-z_{t-1}\Vert^{2} + \frac{\mu a(n)b(n)}{2n} \Vert x-x_{t}\Vert^{2}\right).
\end{equation*}
The objective function in this minimization problem is strongly convex with parameter
$(\frac{L_t}{\alpha_t}+\frac{\mu a(n)b(n)}{n})$. Hence,
\begin{align*}
\MoveEqLeft \langle b(n)\delta_{t},x-x_{t}\rangle +\frac{L_{t}}{2\alpha_{t}}\Vert x-z_{t-1}\Vert^{2} + \frac{\mu a(n)b(n)}{2n} \Vert x-x_{t}\Vert^{2} \\
\geq & \langle b(n)\delta_{t},z_t-x_{t}\rangle +\frac{L_{t}}{2\alpha_{t}}\Vert z_t-z_{t-1}\Vert^{2} + \frac{\mu a(n)b(n)}{2n} \Vert z_t-x_{t}\Vert^{2} \\
& + \frac{L_{t}}{2\alpha_{t}}\Vert x-z_{t}\Vert^{2} + \frac{\mu a(n)b(n)}{2n} \Vert x-z_{t}\Vert^{2}.
\end{align*}
Using this in Lemma~\ref{lemma:OARCD},
\begin{align}
\MoveEqLeft \E[f_{t-1}(y_t)| \H_t] - f_{t-1}(x) \leq \frac{n}{a(n)}\E[\langle \delta_{t},x_{t}-z_{t}\rangle  | \H_t] \nonumber \\
 & - \frac{nL_{t}}{2a(n)b(n)\alpha_{t}}\E[\Vert z_{t}-z_{t-1}\Vert^{2} + \Vert x-z_{t}\Vert^{2} - \Vert x-z_{t-1}\Vert^{2} | \H_t] \nonumber\\
& \vspace{1in} - \frac{\frac{2}{a(n)}L_{t}-L}{2L_{t}^{2}}\E[\Vert\delta_{t}\Vert^{2} | \H_t] - \frac{\mu}{2}\E[\Vert x-z_{t}\Vert^{2}|\H_t], \label{eqn:fx-bound2}
\end{align}
where we have dropped the term $-\frac{\mu}{2} \Vert z_t-x_{t}\Vert^{2}$ from the right hand side without affecting the inequality.

On the other hand,
\begin{align*}
\frac{L_{t}}{2}(\Vert z_{t}-x_{t}\Vert^2 - \Vert z_{t}-y_{t}\Vert^2) &=\frac{L_{t}}{2}(\Vert z_{t}-x_{t}\Vert^{2} - \Vert z_{t}-x_{t}+x_{t}-y_{t}\Vert^{2})\\
& = \frac{L_{t}}{2}\left(2 \langle x_{t}-z_{t}, x_{t}-y_{t} \rangle - \frac{\Vert \delta_{t}\Vert^{2}}{L_{t}^{2}}\right)\\
&= \langle \delta_t, x_{t}-z_{t} \rangle - \frac{\Vert \delta_{t}\Vert^{2}}{2L_{t}}.
\end{align*}
Hence, using the update equation of $x(t)$~(see Algorithm~\ref{algo:OARCD}),
\begin{align*}
\MoveEqLeft \langle \delta_{t}, x_{t}-z_{t}\rangle \\
&= \frac{L_t}{2}(\Vert z_t -(1-\alpha_t)y_{t-1}-\alpha_t z_{t-1}\Vert^2 - \Vert z_t-y_t\Vert^2) + \frac{\Vert\delta_{t}\Vert^2}{2L_t}\\
&=\frac{L_{t}}{2}(\Vert z_{t}-z_{t-1}+(1-\alpha_{t})(z_{t-1}-y_{t-1})\Vert^{2} - \Vert z_{t}-y_{t}\Vert^{2}) + \frac{\Vert \delta_{t}\Vert^{2}}{2L_{t}}\\
&\leq \frac{L_{t}}{2}\left(\frac{\alpha_{t} \Vert z_{t}-z_{t-1}\Vert^{2}}{\alpha_t^2}+(1-\alpha_{t})\Vert z_{t-1}-y_{t-1}\Vert^{2} - \Vert z_{t}-y_{t}\Vert^{2}\right) + \frac{\Vert \delta_{t}\Vert^{2}}{2L_{t}},
\end{align*}
where the inequality follows from convexity of $\Vert \cdot \Vert^{2}$.
Using this inequality in~\eqref{eqn:fx-bound2} with $b(n) = 1$,
\begin{align}
\E[f_{t-1}(y_t)| \H_t] - f_{t-1}(x) \leq & \frac{nL_{t}}{2a(n)}\E[(1-\alpha_{t})\Vert z_{t-1}-y_{t-1}\Vert^{2} - \Vert z_{t}-y_{t}\Vert^{2}|\H_t] \nonumber \\
& + \frac{nL_{t}}{2a(n)\alpha_{t}}\E[\Vert x-z_{t-1}\Vert^{2}- \Vert x-z_{t}\Vert^{2}|\H_t] \nonumber \\
&+\frac{\frac{(n-2)L_{t}}{a(n)} + L}{2L_{t}^{2}} \E[\Vert \delta_{t}\Vert^{2}|\H_t] - \frac{\mu}{2}\E[\Vert x-z_{t}\Vert^{2}|\H_t] \label{eqn:fx-bound3}.
\end{align}
Further, from convexity of $f_{t-1}(\cdot)$,
\begin{align*}
& f_{t-1}(y_{t-1}) - f_{t-1}(y_t) \\
& \leq \langle \nabla f_{t-1}(y_{t-1}), y_{t-1}-y_{t}\rangle\\
&\leq \frac{a(n)\Vert \nabla f_{t-1}(y_{t-1}) \Vert^2}{2n(1-\alpha_{t})(L_{t}-L)} + \frac{n(1-\alpha_{t})(L_{t}-L)\Vert y_{t-1}-y_{t}\Vert^{2}}{2a(n)}\\
& \leq \frac{a(n)R^2}{2n(1-\alpha_{t})(L_{t}-L)} + \frac{n(1-\alpha_{t})(L_{t}-L)\Vert y_{t-1}-x_t+x_t-y_{t}\Vert^{2}}{2a(n)}\\
&= \frac{a(n)R^2}{2n(1-\alpha_{t})(L_{t}-L)} + \frac{n(1-\alpha_{t})(L_{t}-L)\Vert \alpha_t(y_{t-1}-z_{t-1})+ x_t-y_{t}\Vert^{2}}{2a(n)}\\
&=\frac{a(n)R^2}{2n(1-\alpha_{t})(L_{t}-L)} + \frac{n\alpha_t(1-\alpha_{t})(L_{t}-L)\Vert y_{t-1}-z_{t-1}\Vert^2}{2a(n)} + \frac{n(L_{t}-L)\Vert \delta_t \Vert^{2}}{2a(n)L_t^2},
\end{align*}
where the second inequality follows from Young's inequality, the third inequality from
the bound on $\Vert \nabla f_{t-1}(y_{t-1}) \Vert$, the first equality from
the update rule of $x_{t}$ and the second equality from
convexity of $\Vert \cdot \Vert^{2}$. Taking conditional expectation in the above
inequality and adding with~\eqref{eqn:fx-bound3} we get
\begin{align*}
\MoveEqLeft \E[f_{t-1}(y_{t-1})|\H_t] - f_{t-1}(x) \\
  \leq & \frac{a(n)R^{2}}{2n(1-\alpha_t)(L_t-L)} + \frac{nL_t}{2a(n)\alpha_t}\E[\Vert x-z_{t-1}\Vert^2 - \Vert x-z_t\Vert^2|\H_t]  \\
  & - \frac{nL_t}{2a(n)}\E[\Vert z_t-y_t \Vert^2|\H_t] +  \frac{n((1-\alpha_t^2)L_t-\alpha_t(1-\alpha_t)L)}{2a(n)}\E[\Vert y_{t-1}-z_{t-1}\Vert^2|\H_t] \\
  & +\frac{\frac{2(n-1)L_t}{a(n)}+\frac{(a(n)-n)L}{a(n)}}{2L_t^2}\E[\Vert\delta\Vert^2|\H_t] -\frac{\mu}{2}\E[\Vert x-z_t\Vert^2|\H_t].
\end{align*}
Finally we get the desired result by using the bound
$\E[\Vert \delta_{t} \Vert^{2}|\H_t] \leq \frac{a^2(n)R^2}{n}$
and then taking expectation.
\qed
\end{proof}

Let us assume that $x^\ast$ minimizes $\sum_{t=1}^T f_t(y)$~(see~\eqref{eqn:online}).
We now set $\alpha_t, L_t$ and $a(n)$ to obtain best regret bounds.
As in case of SARCD, we first consider $\mu = 0$.
\begin{theorem}
\label{thm:OARCD-general-conv}
Assume that $\mu = 0$, $\Vert \nabla f_t(x)\Vert \leq R$ and $\Vert x^\ast - z_{t}\Vert \leq D$ for $t \geq 1$. Set $a(n) = \sqrt{n}$, $\alpha_{t} = \alpha$ and $L_{t} = \alpha\sqrt{t-1}L + L$, where $\alpha \in (0,1)$ is a constant, and
$G = ((1-\alpha^2)L_2-\alpha(1-\alpha)L)\Vert y_1-z_1\Vert^2$. Then the regret of OARCD can be bounded as
\[
\sum_{t=1}^{T}(\E[f_{t}(y_{t})] - f_{t}(x^\ast)) \leq \left(\frac{2R^2}{\alpha L} + \frac{L D^2}{2}\right)\sqrt{nT} + \frac{R^{2}}{(1-\alpha)\alpha L}\sqrt{\frac{T}{n}} + \left(\frac{(\alpha+1) L D^2}{2\alpha} + \frac{G}{2}\right)\sqrt{n}.
\]
\end{theorem}
\begin{proof}
From Proposition~\ref{prop:OARCD}, using $\alpha_t = \alpha$,
\begin{equation*}
\sum_{t=1}^{T}(\E[f_{t}(y_{t})] - f_{t}(x^\ast)) \leq \E[A + B + C],
\end{equation*}
where
\begin{align*}
A &= \frac{a(n)R^{2}}{2n(1-\alpha)}\sum_{t=1}^T \frac{1}{L_{t+1}-L} + \frac{a(n)R^2}{2n}\sum_{t=1}^T \frac{2(n-1)L_{t+1}  + (a(n)-n)L}{L_{t+1}^2}, \\
B &= \frac{n}{2a(n)\alpha }\sum_{t=1}^T L_{t+1}(\Vert x^\ast-z_t\Vert^2 - \Vert x^\ast-z_{t+1}\Vert^2), \\
C &= \frac{n}{2 a(n)} \sum_{t=1}^T (((1-\alpha^2)L_{t+1}-\alpha(1-\alpha)L)\Vert y_t-z_t\Vert^2 - L_{t+1}\Vert z_{t+1}-y_{t+1} \Vert^2).
\end{align*}
Substituting $L_t = \alpha\sqrt{t-1}L + L$,
\begin{align*}
A &\leq \frac{a(n)R^{2}}{2n(1-\alpha)\alpha L}\sum_{t=1}^T \frac{1}{\sqrt{t}} + \frac{a(n)(n-1)R^2}{n\alpha L}\sum_{t=1}^T \frac{1}{\sqrt{t}} + \frac{a(n)(a(n)-n)R^2}{2n\alpha^2 L}\sum_{t=1}^T \frac{1}{t}\\
& \leq \frac{a(n)R^{2}\sqrt{T}}{n(1-\alpha)\alpha L} + \frac{2a(n)R^2\sqrt{T}}{\alpha L} +  \frac{a(n)(a(n)-n)R^2\ln(T)}{2n\alpha^2 L},\\
B & = \frac{n}{2a(n)\alpha}\sum_{t=2}^T \Vert x^\ast-z_t\Vert^2 (L_{t+1} -L_t)+ \frac{n L_2 \Vert x^\ast-z_1\Vert^2}{2a(n)\alpha} - \frac{n L_{T+1} \Vert x^\ast-z_{T+1}\Vert^2}{2a(n)\alpha} \\
& \leq \frac{nL D^2 \sqrt{T}}{2a(n)} + \frac{n(\alpha+1) L D^2 }{2a(n)\alpha}.
\end{align*}
Further,
\begin{align*}
C =& \frac{n((1-\alpha^2)L_2-\alpha(1-\alpha)L)\Vert y_1-z_1\Vert^2}{2 a(n)} - \frac{nL_{T+1}\Vert z_{T+1}-y_{T+1} \Vert^2}{2a(n)}\\
   & + \frac{n}{2 a(n)} \sum_{t=2}^T \Vert y_t-z_t\Vert^2 ((1-\alpha^2)L_{t+1}-\alpha(1-\alpha)L - L_t).
\end{align*}
However, for $t \geq 2$,
\begin{align*}
\MoveEqLeft (1-\alpha^{2})L_{t+1}-\alpha(1-\alpha)L - L_{t} \\
&= (1-\alpha^{2})\alpha L \sqrt{t} + (1-\alpha^2) L - \alpha(1-\alpha) L - \alpha L \sqrt{t-1} -L \\
           & \leq \alpha L (\sqrt{t} - \sqrt{t-1}) - \alpha L \leq 0.
\end{align*}
Hence,
\[
C \leq \frac{nG}{2a(n)}.
\]
So, setting $a(n) = \sqrt{n}$ and using the bounds on $A, B$ and $C$,
\begin{equation*}
\sum_{t=1}^{T}(\E[f_{t}(y_{t})] - f_{t}(x^\ast)) \leq \frac{2R^2\sqrt{nT}}{\alpha L} + \frac{L D^2 \sqrt{nT}}{2} + \frac{R^{2}\sqrt{T}}{(1-\alpha)\alpha L\sqrt{n}} + \frac{(\alpha+1) L D^2\sqrt{n} }{2\alpha} + \frac{G\sqrt{n}}{2},
\end{equation*}
which is the desired bound.
\qed
\end{proof}

Finally, we consider the case when $\mu > 0$.
\begin{theorem}
\label{thm:OARCD-strongly-conv}
Assume that $\mu > 0$, $\Vert \nabla f_t(x)\Vert \leq R$ and $\Vert x^\ast - z_{t}\Vert \leq D$ for $t \geq 1$. Set $a(n) = n$, $\alpha_{t} = \alpha$ and $L_{t} = \alpha\mu t  + L$, where $\alpha \in (0,1)$ is a constant, and $G = ((1-\alpha^2)L_2-\alpha(1-\alpha)L)\Vert y_1-z_1\Vert^2$. Then the regret of OARCD can be bounded as
\[
\sum_{t=1}^{T}(\E[f_{t}(y_{t})] - f_{t}(x^\ast)) \leq  \frac{R^{2}\ln(T+1)}{2(1-\alpha)\alpha \mu} + \frac{nR^2\ln(T+1)}{\alpha \mu} +  \frac{D^2 }{2\alpha}\left(2\alpha\mu + L \right) + \frac{G}{2}.
\]
\end{theorem}
\begin{proof}
As in the proof of Theorem~\ref{thm:OARCD-general-conv}, we can write
\begin{equation*}
\sum_{t=1}^{T}(\E[f_{t}(y_{t})] - f_{t}(x^\ast)) \leq \E[A + B + C],
\end{equation*}
where $A, C$ are exactly same as before and
\begin{equation*}
B = \frac{n}{2a(n)\alpha }\sum_{t=1}^T L_{t+1}(\Vert x^\ast-z_t\Vert^2 - \Vert x^\ast-z_{t+1}\Vert^2) - \frac{\mu}{2}\sum_{t=1}^T \Vert x^\ast-z_{t+1}\Vert^{2}.
\end{equation*}
Substituting $L_t = L_{t} = \alpha\mu t + L$,
\begin{align*}
A \leq & \frac{a(n)R^{2}}{2n(1-\alpha)\alpha \mu}\sum_{t=1}^T \frac{1}{t} + \frac{a(n)(n-1)R^2}{n\alpha \mu}\sum_{t=1}^T \frac{1}{t} + \frac{a(n)(a(n)-n)R^2 L}{2n\alpha^2 \mu^2}\sum_{t=1}^T \frac{1}{t^2}\\
\leq & \frac{a(n)R^{2}\ln(T+1)}{2n(1-\alpha)\alpha \mu} + \frac{a(n)R^2\ln(T+1)}{\alpha \mu} + \frac{a(n)(a(n)-n)R^2 L}{2n\alpha^2 \mu^2}\sum_{t=1}^T \frac{1}{t^2},\\
B  = &\frac{n L_2 \Vert x^\ast-z_1\Vert^2}{2a(n)\alpha} + \frac{n}{2a(n)\alpha}\sum_{t=2}^T \Vert x^\ast-z_t\Vert^2 (L_{t+1} -L_t) - \frac{n L_{T+1} \Vert x^\ast-z_{T+1}\Vert^2}{2a(n)\alpha} \\
   & - \frac{\mu}{2}\sum_{t=1}^T \Vert x^\ast-z_{t+1}\Vert^{2}\\
\leq & \frac{n D^2 }{2a(n)\alpha}\left(2\alpha\mu + L \right)+ \frac{n}{2a(n)\alpha}\sum_{t=2}^T \Vert x^\ast-z_t\Vert^2 \left(\alpha \mu - \alpha \mu \frac{a(n)}{n}\right).
\end{align*}
Further,
\begin{align*}
C =& \frac{n((1-\alpha^2)L_2-\alpha(1-\alpha)L)\Vert y_1-z_1\Vert^2}{2 a(n)} - \frac{nL_{T+1}\Vert z_{T+1}-y_{T+1} \Vert^2}{2a(n)}\\
   & + \frac{n}{2 a(n)} \sum_{t=2}^T \Vert y_t-z_t\Vert^2 ((1-\alpha^2)L_{t+1}-\alpha(1-\alpha)L - L_t).
\end{align*}
However, for $t \geq 2$,
\begin{align*}
(1-\alpha^{2})L_{t+1}-\alpha(1-\alpha)L - L_{t} &\leq (L_{t+1}-L_{t}) - \alpha^{2}L_{t+1} - \alpha(1-\alpha)L \\
           & = \alpha \mu - \alpha^2 (\alpha \mu (t+1)+L)-\alpha(1-\alpha)L \\
           & = \alpha\mu - \alpha^2 (\alpha \mu (t+1))-\alpha L < 0.
\end{align*}
Hence,
\[
C \leq \frac{nG}{2a(n)}.
\]
So, setting $a(n) = n$ and using the bounds on $A, B$ and $C$,
\begin{equation*}
\sum_{t=1}^{T}(\E[f_{t}(y_{t})] - f_{t}(x^\ast)) \leq \frac{R^{2}\ln(T+1)}{2(1-\alpha)\alpha \mu} + \frac{nR^2\ln(T+1)}{\alpha \mu} +  \frac{D^2 }{2\alpha}\left(2\alpha\mu + L \right) + \frac{G}{2},
\end{equation*}
which is the desired bound.
\qed
\end{proof}

\section{Numerical Evaluation}
The algorithm proposed in this paper OARCD gives an improvement over ORBCD~\cite{IEEEhowto:kopka15,IEEEhowto:kopka16} as shown in red in the figures 1,2 and 3. The regret is much lower for both classification as well as regression. Table 1 shows the dataset description taken from UCI Repository. We have performed several experiments on RCV1 dataset also. We have observed that regularization is needed for the well posed-ness of the problem. Adaptive coordinated descent method is also used for setting the learning rates as a function of the sum of the gradients which reduced the regret significantly. We have observed that normalization and choosing L as sum of the squares of maximum features are related in the sense that if we do not normalize but set L to be the  sum of the squares of maximum features then also it gives the same result as we obtained after normalizing and setting L to be the number of features. We don't require per coordinate Lipschitz continuity if normalization is done. Online gradient methods will not work in cases where number of features are more. In the dataset 3 and 4 as shown in Table 1 and 2, OGD~\cite{IEEEhowto:kopka3} and SAGE~\cite{IEEEhowto:kopka4} took more than 1 minute while OARCD proposed in this paper took only 6 seconds to compute the regret. Also, OARCD shows improvent in accuracy and performs less number of mistakes than ORBCD~\cite{IEEEhowto:kopka15,IEEEhowto:kopka16} as shown in Table 2. When we add more number of examples, we see that the regret becomes constant since the difference between the best algorithm upto time t and the online algorithm becomes less. Figure 4 shows the comparison between the loss of APPROX~\cite{IEEEhowto:kopka11} and SARCD proposed in this paper.\\

\begin{figure}
\centering
\begin{minipage}{.6\textwidth}
  \centering
  \includegraphics[width=1\linewidth]{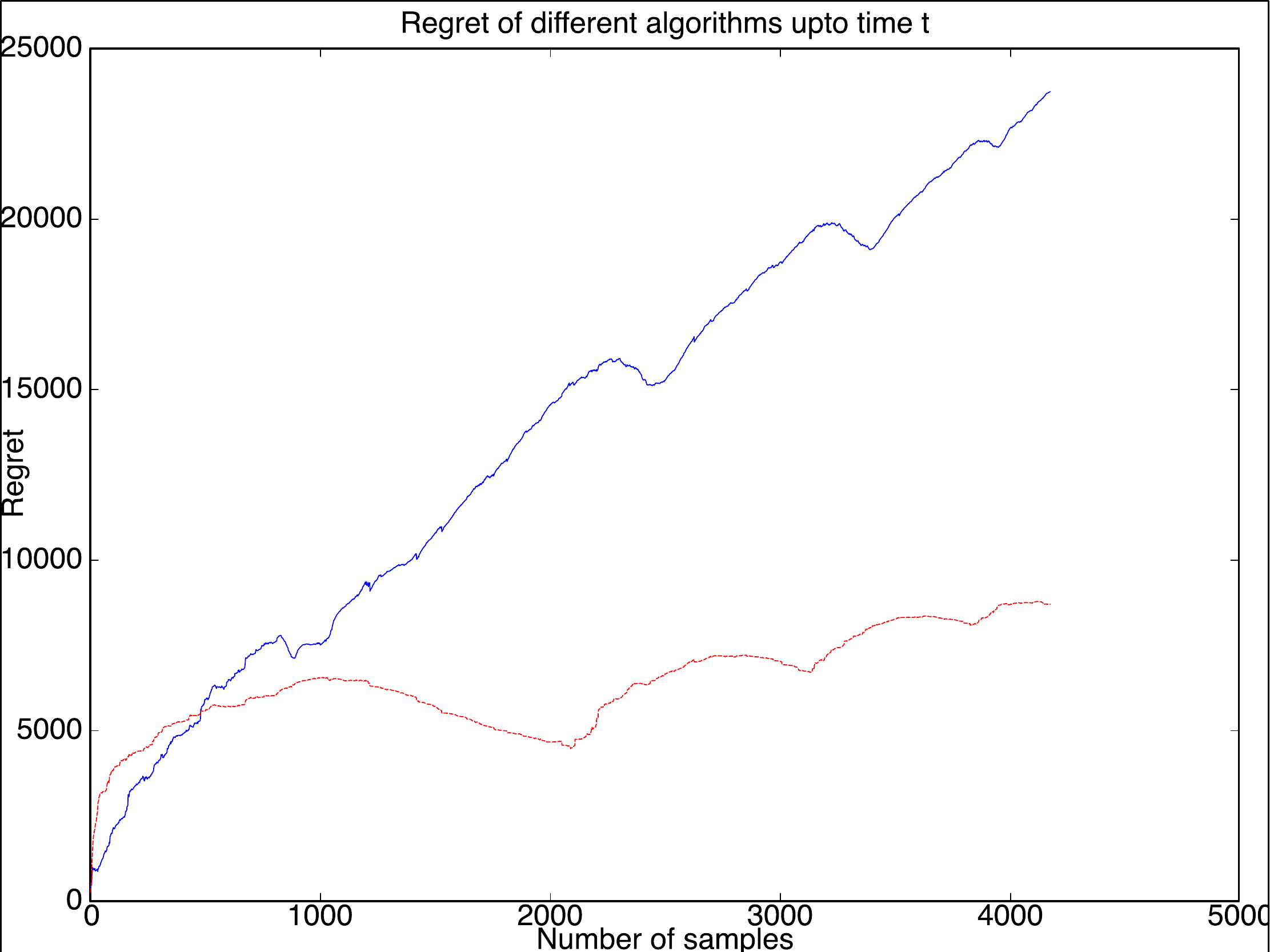}
  \caption{Regret comaprison on abalone dataset}
  \label{fig:test1}
\end{minipage}%
\begin{minipage}{.6\textwidth}
  \centering
  \includegraphics[width=1\linewidth]{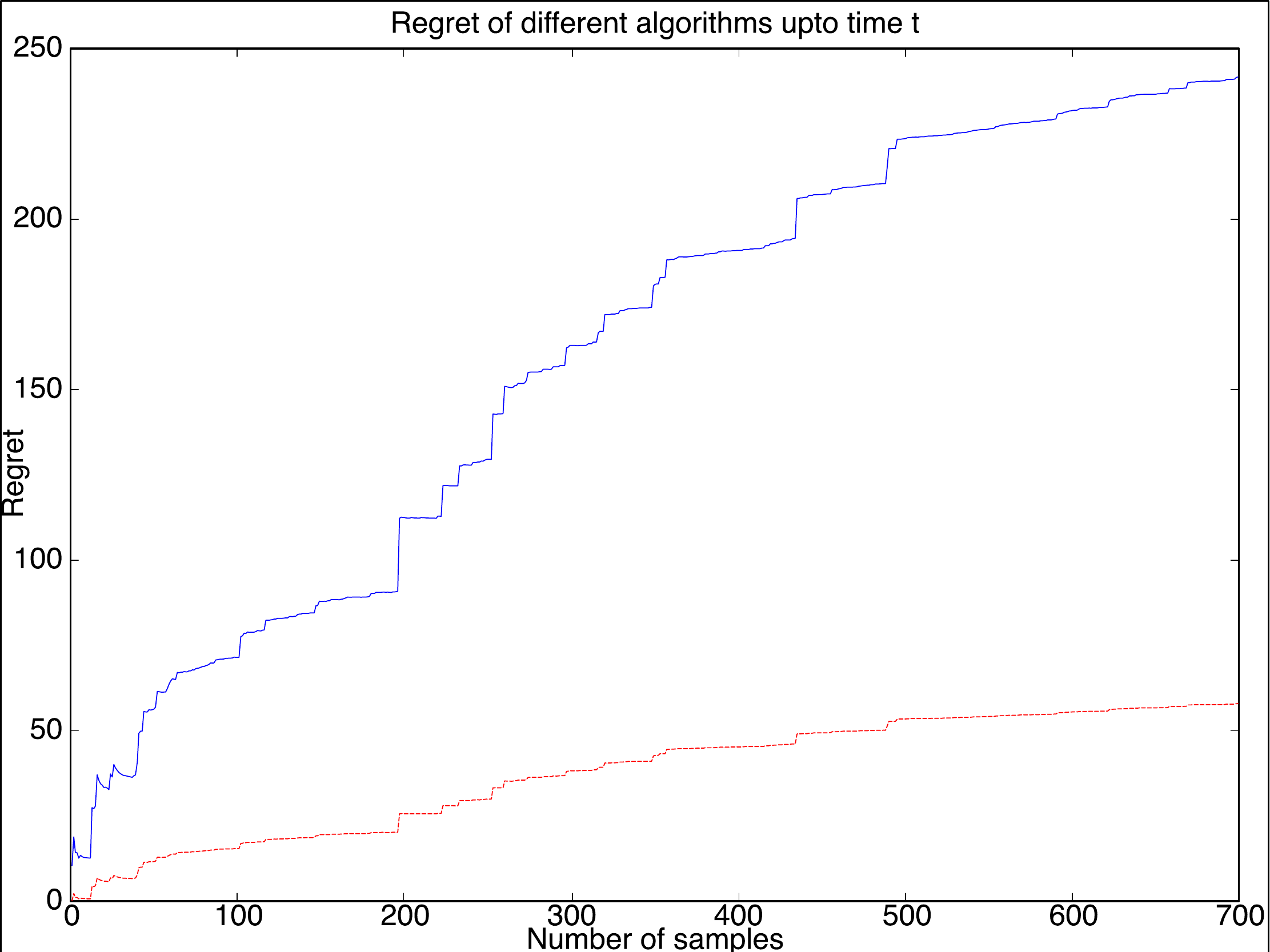}
  \caption{Regret comaprison on breastcancer dataset}
  \label{fig:test2}
\end{minipage}
\end{figure}

\begin{figure}
\centering
\begin{minipage}{.6\textwidth}
  \centering
  \includegraphics[width=1\linewidth]{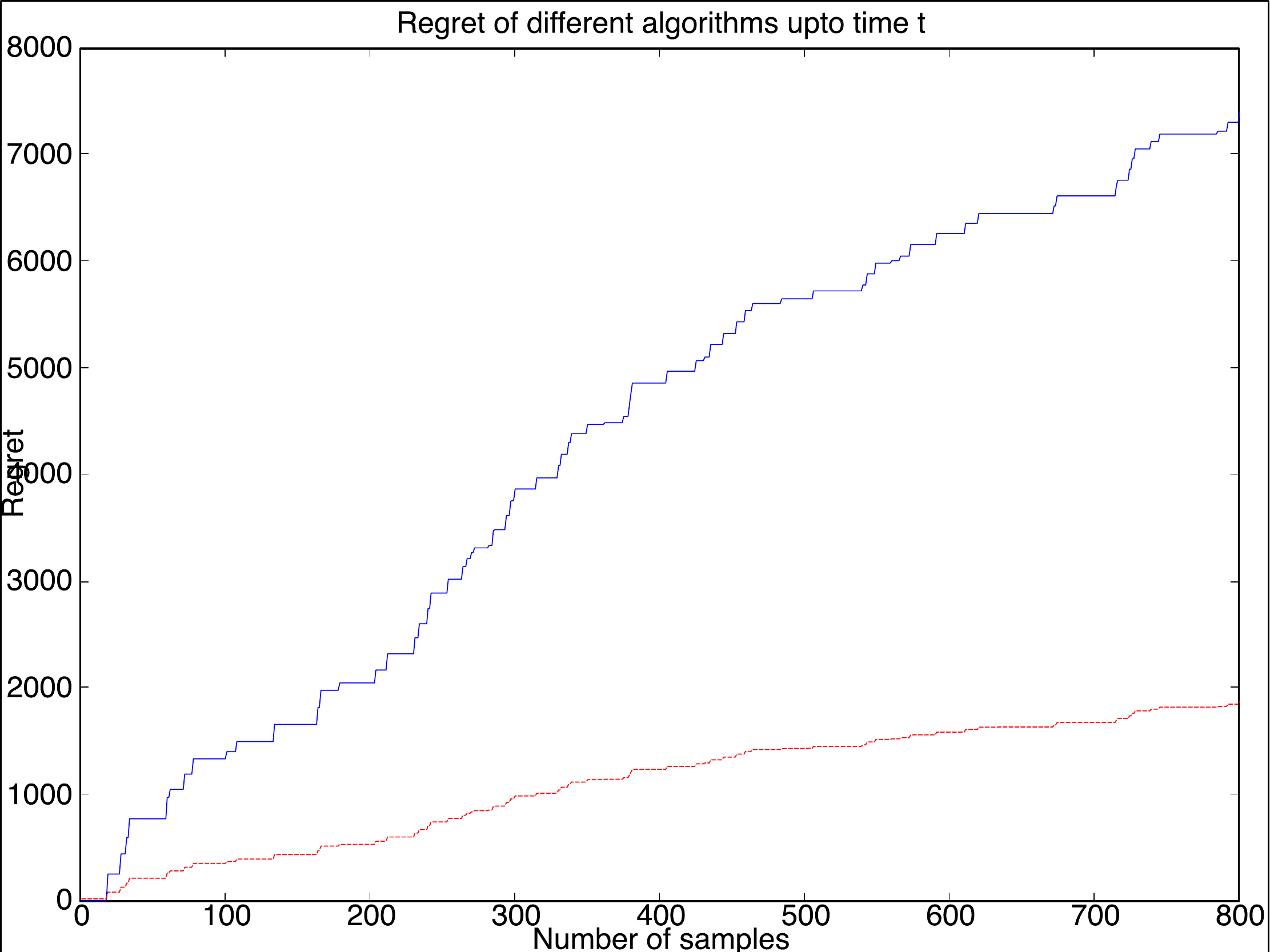}
  \caption{Regret comaprison on dorothea dataset}
  \label{fig:test1}
\end{minipage}%
\begin{minipage}{.6\textwidth}
  \centering
  \includegraphics[width=1\linewidth]{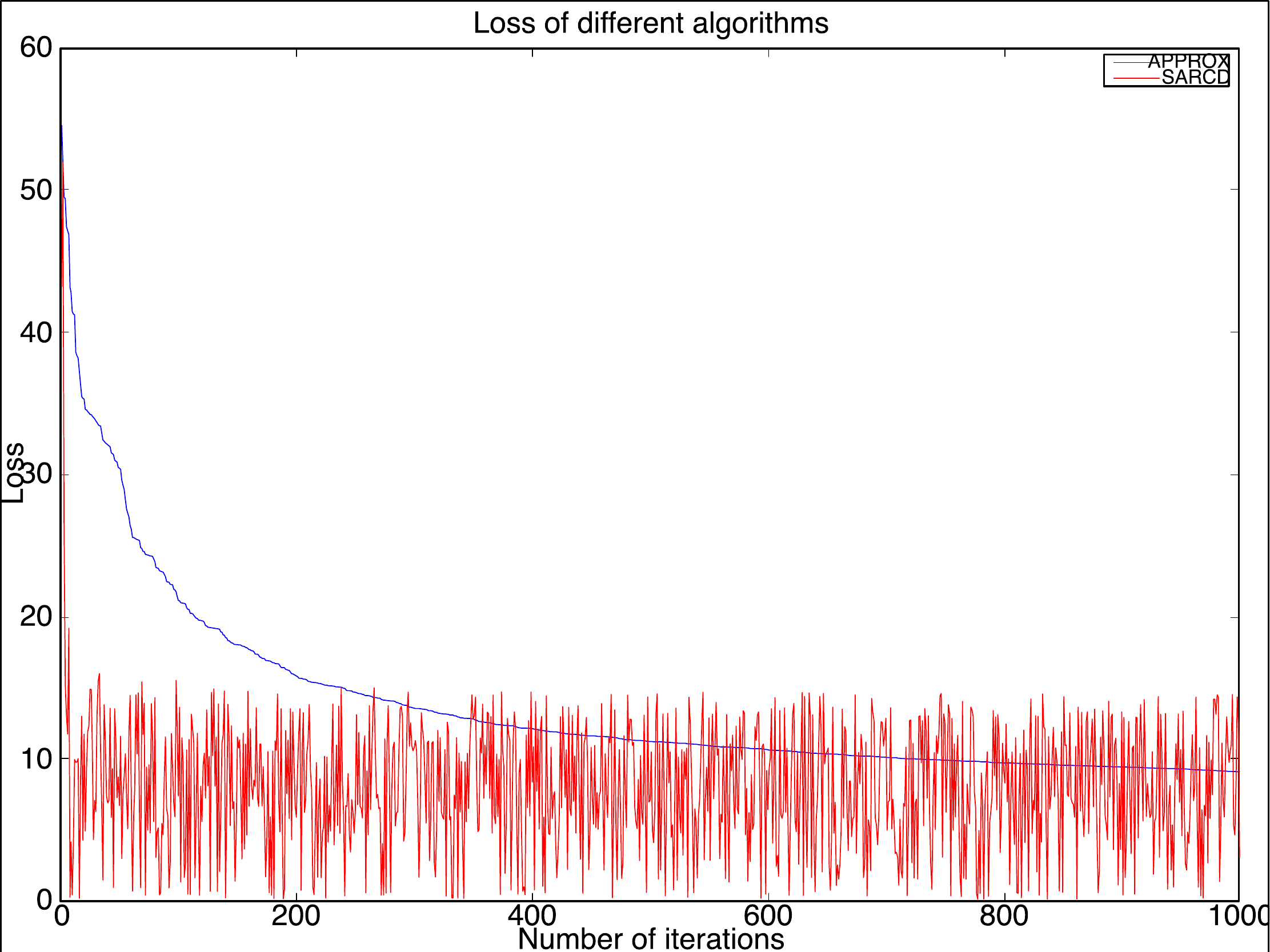}
  \caption{Loss of different algorithms on abalone dataset}
  \label{fig:test2}
\end{minipage}
\end{figure}

\begin{table}
\parbox{.35\linewidth}{
\centering
\begin{tabular}{|l|l|l|l|c|}
\hline
Datasets & \# features & \# examples   & Type           \\
\hline\hline
abalone  & 7           & 4177          & Regression     \\
breast cancer & 9      & 699       & Classification \\
dorothea & 100000      & 1950          & Classification \\
RCV1     & 47236       & 20242/677,399 & Classification\\
\hline
\end{tabular}
\caption{Dataset Description}
}
\hfill
\parbox{.35\linewidth}{
\centering
\begin{tabular}{|l|l|l|l|c|}
\hline
Algorithm        & Accuracy & \# Mistakes \\
\hline\hline
OARCD on abalone             & 91.32         & - \\
OARCD on breastcancer & 94.423462\%                            &  22       \\
OARCD on dorothea & 90.25\%                        & 78          \\
OARCD on RCV1 & 89\%                        & 115        \\
\hline
\end{tabular}
\caption{Accuracy and number of mistakes}
}
\end{table}

\section{Conclusion}
We have proposed two accelerated randomized coordinate descent algorithms 
for stochastic optimization and online learning, respectively.
Our algorithms exhibit performance as good as the best known
randomized coordinate descent algorithms and yield strictly better 
regret bounds in case of online learning.

Our ongoing and future work entails extending these algorithms to 
regularized loss functions. We would like to investigate
adaptation of feature selection probabilities to coordinate wise
smoothness parameters. We would also like to consider online learning problems where 
update of model parameters takes considerable time, and so, the updated
parameters are available only after a certain, potentially random, number
of data samples have passed. 

\subsubsection*{Acknowledgments:}
The second author acknowledges support of INSPIRE Faculty Research Grant~(DSTO-1363).

\end{document}